\documentclass[letterpaper]{article} 
\usepackage{aaai19}  
\usepackage{times}  
\usepackage{helvet} 
\usepackage{courier}  
\usepackage[hyphens]{url}  
\usepackage{graphicx} 
\usepackage{graphicx}  
\usepackage{graphicx}
\usepackage{amssymb}
\usepackage{amsmath}
\usepackage[ruled,linesnumbered,vlined]{algorithm2e}

\usepackage{subcaption} 
\usepackage{booktabs} 

\usepackage{todonotes}
\usepackage{lipsum} 

\newcommand{\beq}{\begin{equation}}
\newcommand{\eeq}{\end{equation}}
\newtheorem{theorem}{Theorem}

\newtheorem{definition}{Definition}
\newenvironment{proof}{\noindent {\bf \em Proof:}}{\hfill $\Box$}

\newcommand{\dsafe}{d_{\mathit{safe}}}
\newcommand{\sroot}{s_{\mathit{root}}}
\newcommand{\starget}{s_{\mathit{target}}}

\newcommand{\citep}[1]{\citeauthor{#1}~(\citeyear{#1})}

\let\oldnl\nl
\newcommand{\nonl}{\renewcommand{\nl}{\let\nl\oldnl}}

\title{Improved Safe Real-time Heuristic Search}

\author{Bence Cserna \and Kevin C. Gall \and Wheeler Ruml\\
Department of Computer Science \\
University of New Hampshire \\
Durham, NH, 03824, USA \\
\\ {\tt bence} at {\tt cs.unh.edu}, {\tt kcg245} at {\tt gmail.com}, {\tt ruml} at {\tt cs.unh.edu}}

\begin{document}
\nocopyright
\maketitle

\begin{abstract}

A fundamental concern in real-time planning is the presence of dead-ends in the state space, from which no goal is reachable.  Recently, the SafeRTS algorithm was proposed for searching in such spaces.  SafeRTS exploits a user-provided predicate to identify safe states, from which a goal is likely reachable, and attempts to maintain a backup plan for reaching a safe state at all times.  In this paper, we study the SafeRTS approach, identify certain properties of its behavior, and design an improved framework for safe real-time search.  We prove that the new approach performs at least as well as SafeRTS and present experimental results showing that its promise is fulfilled in practice.


   
   %
   
\end{abstract}

\section{Introduction}


Systems that interact with the external physical world often must be controlled in real time. Examples include systems that interact with humans and robotic systems, such as autonomous vehicles. 
In this paper, we address real-time planning, where the planner must return the next action for the system to take within a specified wall-clock time bound. 

Providing real-time heuristic search algorithms that are complete in domains with dead-end states is a challenging problem.
Traditional real-time planners are inherently incomplete due to the limited time available to make a decision even when the state space is fully observable and the actions are deterministic.
%
%
%
%
\citep{cserna:ade} proposed the first real-time heuristic search method, SafeRTS, that is able to reliably reach a goal in domains with dead-ends. Prior real-time methods focus their search effort on a single objective that minimizes the cost to reach the goal. A single objective is insufficient to provide completeness and minimize the time to reach the goal as these often contradict each other.
Thus, SafeRTS distributes the available time between searches optimizing the independent objectives of safety and finding the goal.


The contribution of this work is four-fold.
First, we argue that benchmark domains used for real-time planning may not be good indicators of performance in the context of safe real-time planning. We present a new set of benchmarks that overcome the deficiencies of previous benchmark domains.
Second, we show how to utilize meta information presented by safety oriented real-time search methods to reduce redundant expansions during both the safety and goal-oriented searches. This improvement marginally reduces the goal achievement time (GAT) while it does not increase the space and time complexity of the safe real-time search method.
Third, we prove inefficiencies in the approach taken by SafeRTS by examining properties of local search spaces and the changing priority of which nodes to prove safe as an LSS grows.
Lastly, we introduce a new framework for safe real-time heuristic search that utilizes the time bound unique to real-time planning.
This framework follows the same basic principle of search effort distribution as SafeRTS but does so more efficiently. We empirically demonstrate the potential of the new framework.

\section{Preliminaries}



Heuristic search methods use a heuristic function $h(s)$ to estimate the cheapest cost $c(s, s_{goal})$ to move from any state $s$ to a goal state $s_{goal}$. A* and other offline or anytime methods construct a full path from the agent start state $s_{start}$ to $s_{goal}$ before committing the agent along any path, however real-time search methods conform to hard time bounds in which they must commit the agent to actions even if no complete path to a goal has been discovered.

A* achieves optimality by expanding nodes in the search graph ordered on $f(s) = g(s) + h(s)$ where $g(s) = c(s_{start}, s)$. In the real-time setting, $g$ as a function relative to $s_{start}$ becomes problematic. As the agent is committed to actions that lead it further away from $s_{start}$, the notion of $g(s)$ becomes less relevant to a state $s$. Agent-centered real-time search is a form of real-time search that focuses exploration and learning in the immediate area around the agent, often using a bounded lookahead to construct a ``Local Search Space." The pioneering work on LRTA* \cite{korf:rth} describes techniques for planning under resource constraints and updating, or ``learning," heuristic values around the agent as it moves through the state space. Our non-safety-oriented baseline, LSS-LRTA*, is built on these core ideas. Pseudocode is sketched in Algorithm \ref{lss-lrta}.

\begin{algorithm}
\SetKwInOut{Input}{Input}

\Input{$s_{root}, bound$}
$s_{current} \gets s_{root}$ \\
\While{$s_{current}$ is not goal}{
    perform $bound$ expansions of A* with $s_{current}$ as root \\
    if $open$ becomes empty, terminate with failure \\
    $s \gets $ node on $open$ with lowest $f$ \\
    update $h$ values of nodes in $closed$ \\
    commit to actions along path from $s_{current}$ to $s$ \\
    $s_{current} \gets s$
}

\caption{LSS-LRTA* \label{lss-lrta}}
\end{algorithm}

The algorithm proceeds in 2 phases: planning and learning. The planning phase is similar to A*: expanding nodes in best-first order, preferring low $f$ values where the root of the search $s_{root}$ is set as the agent's current state. To obey the real-time bound, only a pre-specified number of nodes are expanded, forming a local search space. In the learning phase, a Dijkstra-like propagation updates the heuristic values of all expanded nodes backwards from the search frontier. The agent then commits to all the actions leading to the top node on the A* open list. If one or more goal states are discovered during planning, the agent commits to the best path to a goal. If the open list becomes empty, a goal state is not reachable and we say the agent has failed.

LSS-LRTA* is complete if the domain is finite, $h$ is consistent and the cost of every action is bounded from below by a constant \cite{koenig:crt}. It works in directed state spaces with non-uniform costs and can handle planning in initially unknown environments.

\subsection{Safety in Heuristic Search}

A dangerous limitation of most real-time search methods is that in directed domains, no resources are spent on ensuring that the path being committed does not lead to a dead-end. If a terminal state $s_t$ (i.e. one with no successors) is just beyond the expanded search frontier, the agent may still commit actions toward an immediate predecessor of $s_t$.

In their work on safe real-time search, \citep{cserna:ade} introduced the notion of safety as a way of evaluating which states are less likely to lead to dead-ends. Here we expand on this notion and provide formalized definitions for safety concepts.

\begin{definition}
Any path $\mathrm{p}$ in the set of all possible paths through the state space $\mathrm{P}$ is a sequence of nodes such that node $\mathrm{n_{i+1}}$ is a successor of $n_i$ terminating in some arbitrary node $\mathrm{n_k}$.
\begin{align*}
    p =& \langle n_1, n_2, \dots n_k \rangle \in P \Leftrightarrow \\
    &\forall_{i \in [1 \dots k]} n_i \in N \wedge \forall_{i \in [1 \dots k - 1]} n_{i+1} \in \textit{successor}(n_i) 
\end{align*}

\end{definition}

\begin{definition}
A node n is \textup{safe} iff there exists a path $\mathrm{p}$ that begins at n and ends with a goal.
\[ n \text{ is } \mathit{safe \: iff} \, \exists \, p \in P : \, p_1 = n \wedge \mathit{isGoal}(p_{|p|}) \]
\end{definition}

\begin{definition}
A \textup{dead-end} node is a node $\mathrm{n}$ from which there is no path $\mathrm{p}$ to a goal.
\[ n \text{ is } \text{dead-end} \: \mathit{iff} \, \neg \, \exists \, p \in P : p_1 = n \land \mathit{isGoal}(p_{|p|}) \]
\end{definition}

\noindent A node that is likely to be safe according to some criterion we will denote $\mathit{safe}_L$.

\begin{definition}
A safety predicate $\mathrm{f_{safe}}$ determines whether a given node is likely to be safe or its safety property is unknown.
\[f_{\mathit{safe}}:  N \rightarrow \{\mathit{safe}_L, \mathit{unknown}\}\]
\end{definition}

\noindent $f_\mathit{safe}$ is a user provided function without any particular guarantees, merely to guide the search algorithm towards states that are \textit{likely} to not lead to dead-end states.

\begin{definition}
A safety predicate $\mathrm{f_{safe}}$ is strong iff there exist a path $\mathrm{p}$ to the goal from every node $\mathrm{n}$ that is marked $\mathrm{safe_L}$ by the function.
\begin{align*}
\forall n \in& N : f_{\mathit{safe}}(n) \rightarrow \mathit{safe}_L \\
&\exists \, p \in P : p_1 = n \wedge \mathit{isGoal}(p_{|p|})
\end{align*}

\end{definition}

\begin{definition}
A node {\normalfont n} is \textup{explicitly safe} if {\normalfont $f_\mathit{safe}(n) = \mathit{safe}_L$}. A node $n'$ is \textit{implicitly safe} if it is a predecessor of a \textit{safe} node $n$.
\[ (f_\mathit{safe}(n) = \mathit{safe}_L) \wedge \exists \, p \in P : p_1 = n' \wedge n \in p \]
\end{definition}

\noindent We refer to both \textit{explicitly safe} and \textit{implicitly safe} nodes simply as \textit{safe} when the distinction is unimportant.

\begin{definition}
A safety proof of a node {\normalfont n} is a path {\normalfont p} that begins at {\normalfont n} and ends at a safe node.
\[ \mathit{proof}(n) = p \in P \: : p_1 = n \land f_\mathit{safe}(p_{|p|}) = \mathit{safe}_L \]

\end{definition}

\begin{definition}
$\mathrm{proof^*(n)}$ is an optimal safety proof of $\mathrm{n}$ if it has minimum number of states. 
\end{definition}

\begin{definition}
A safety heuristic function $\mathrm{d_{safe}(n)}$ is a function that estimates the minimum distance in state transitions between $\mathrm{n}$ and any \textup{safe} node $\mathrm{n'}$.
\[ d_\mathit{safe} : N \rightarrow \mathbb{N}^0 \]
\end{definition}

\noindent $d_{safe}$ is a user-defined heuristic function that does not require any suboptimality bound on its estimate of the distance to a safe state.

\subsection{SafeRTS}

 SafeRTS \cite{cserna:ade} is roughly similar to LSS-LRTA*, but with key differences in how resources are allocated in planning and how target states are selected. Pseudocode is provided in Algorithm \ref{alg:srts}.

\begin{algorithm}[t]
\SetKwInOut{Input}{Input}
\Input{$s_{root}, bound$}

\While{$\mathit{s_{root}} \neq \mathit{s_{goal}}$}{
$C \leftarrow \emptyset $\\
$b \leftarrow 10$ $\lhd$ initialize expansion budget\\
\While{expansion limit {\em bound} is not reached}{
  perform \textsc{Astar} for $b$ expansions \\
    \Indp
    \nonl adding any safe node discovered to $C$ \\
    \Indm
  perform \textsc{Best-First Search} on $\mathit{d_{safe}}$ \\
    \Indp 
    \Indm

  \If{such $c$ found}{
  cache safety of nodes in path from $t$ to $c$ \\
  $b \leftarrow 10$ $\lhd$ reset budget \\
  $C \leftarrow C \cup \{t\}$ \\
  }
  \Else {
  $b \leftarrow 2 * b$ $\lhd$ increase budget\\
  }
}
\For{$c \in C$}{ 
propagate safety to ancestors of $c$ \\
choose node $s$ in {\em open} with lowest $f$ value that has $\mathit{s_{safe}}$ safe predecessor \\
}
\If{such $s$ and $\mathit{s_{safe}}$ exists}{
$\mathit{s_{target}} \leftarrow \mathit{s_{safe}}$ \\
}
\ElseIf{identity action is available at $\sroot$}{
$\starget \leftarrow \sroot$  $\lhd$ apply identity action\\
}
\Else{
\textsc{Terminate} $\lhd$ no safe path is available 
}
use \textsc{Dijsktra} to update $h$ values of the nodes \\
move the agent along the path from $\sroot$ to $\starget$\\
$\sroot \leftarrow \starget$
}
\caption{SafeRTS}
\label{alg:srts}
\end{algorithm}

SafeRTS splits planning resources between exploration for a goal via best-first search on $f$ and attempting to prove promising nodes as safe using $\dsafe$. The planning phase alternates between these two tasks: first, nodes are expanded as in A* to some expansion limit $b$. Then, the node on top of the open list is selected as the target of the safety proof. Nodes are expanded using a best-first search on $\dsafe$, starting at the target node until the same node expansion limit $b$ is reached or until a \textit{safe} node is discovered. Nodes expanded by the proving stage are not added to the search tree because their $g$ values are not based on the agent's location. If no \textit{safe} node was discovered in the proving stage, the resource limit $b$ is doubled and the algorithm switches back to goal-finding with expansions on $f$. Once $bound$ total expansions occur between both exploration and proving stages, the algorithm shifts to the learning phase.

In the learning phase, the safety of all discovered safe nodes is propagated back to their ancestors. Then $h$ values are propagated back through the local search space in exactly the same manner as in LSS-LRTA*. Actions are selected based on a strategy referred to as ``safe-toward-best" whereby the agent would select as its target the deepest \textit{safe} node along the path to the node on the open list that a) has a \textit{safe} predecessor, and b) has the best $f$ value of all other nodes on the frontier with a \textit{safe} predecessor. If no nodes on the open list have a \textit{safe} predecessor, the agent takes an ``identity action" defined as an action which transitions to the exact same state, if such an action exists.

The approach of using half the expansions on proving safety instead of exploration means that the search tree of SafeRTS does not go as deep as that of LSS-LRTA*. Empirical results showed that the tradeoff was well worth the effort as SafeRTS was able to avoid dead-ends at far greater rates than the baseline LSS-LRTA*. However, the technique of switching between exploration and proving stages within the planning phase is inefficient in that nodes may be expanded in the proving stage that will be expanded in a subsequent exploration stage. Below, we will present theorems supporting this assertion and we explore alternatives. But before we discuss and evaluate new algorithms in detail, we introduce a new and more efficient benchmark domain.


\section{A Benchmark Domain for Evaluating Safety}



Real-time planning algorithms construct a solution iteratively and start to move the agent immediately after the first completed iteration. In the chain of decisions, the starting point of each decision is the result of all prior decisions. We argue that it is important to balance the impact of each decision on the overall GAT. 

\begin{figure}[t]
\centering
  \includegraphics[width=0.8\linewidth]{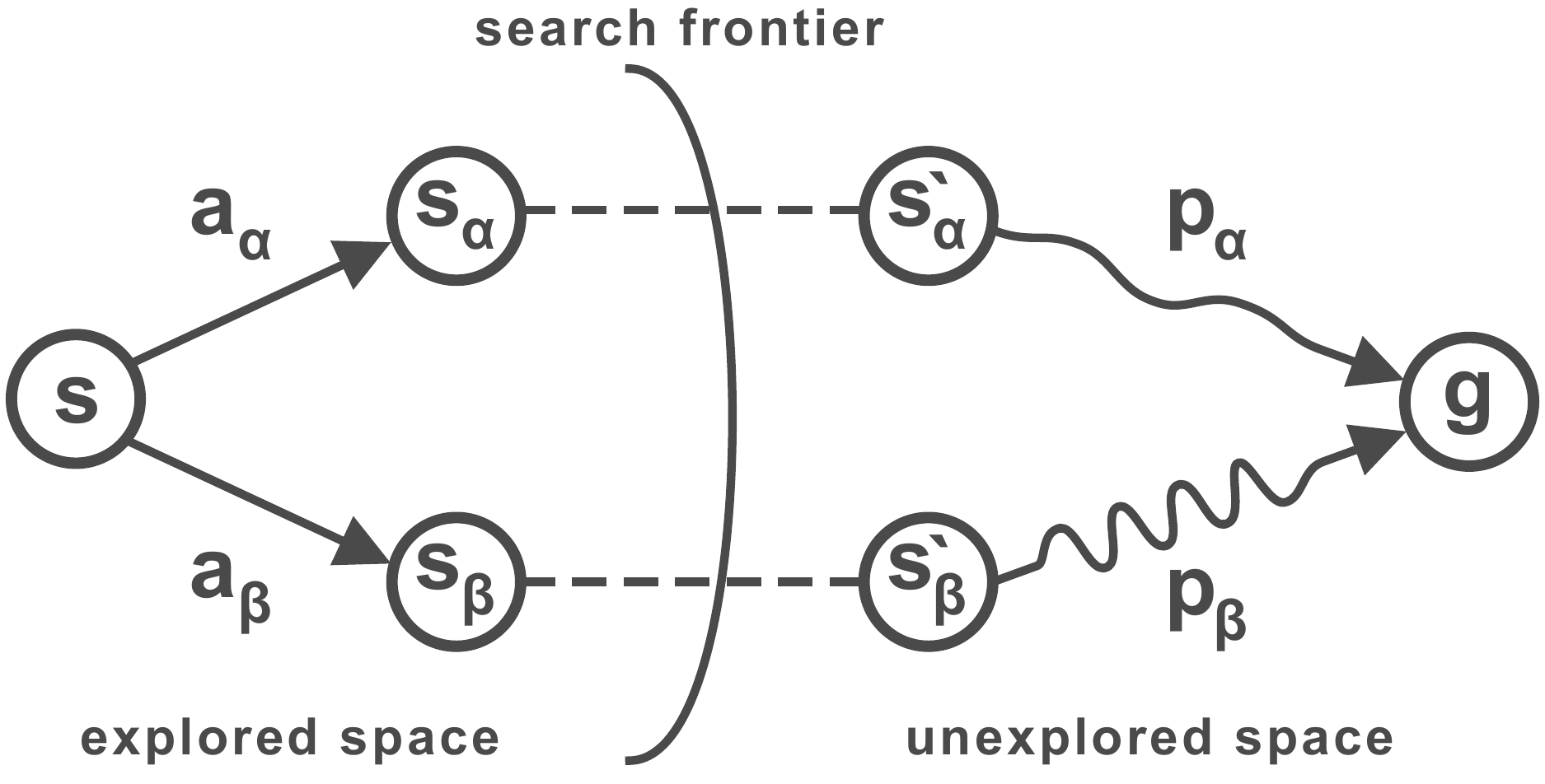}
\caption{Demonstration of actions with long term effects.}
\label{fig:long_term_effect}
\end{figure}

Consider the example domain in Figure~\ref{fig:long_term_effect}. The circles mark the states, the straight arrows the actions, the dotted lines two identical segments, and the squiggle arrows long segments. The segments consist of a sequence of states connected by actions. The agent is currently at state $s$ and the planner has to decide between selecting action $a_\alpha$ and $b_\beta$ which lead to states $s_\alpha$ and $s_\beta$ respectively. 
Given the real-time setting, the planner has limited time to inspect the possibilities beyond these states and would not be able to switch between the $\alpha$ and $\beta$ path after the first step.
Assuming that the dotted segments and the leading actions are identical and that the exploration does not reach $s_\alpha^{'}$ and $s_\beta^{'}$ states, deciding between the two alternative actions is not informed because $p_\beta$ is much cheaper than $p_\alpha$. 
This arbitrary decision would penalize or reward the planner despite the fact that it does not reflect intelligent decision-making. In domains where similar settings exist single actions have a disproportional impact on the GAT.

\begin{figure}[t]
\centering
  \includegraphics[width=\linewidth]{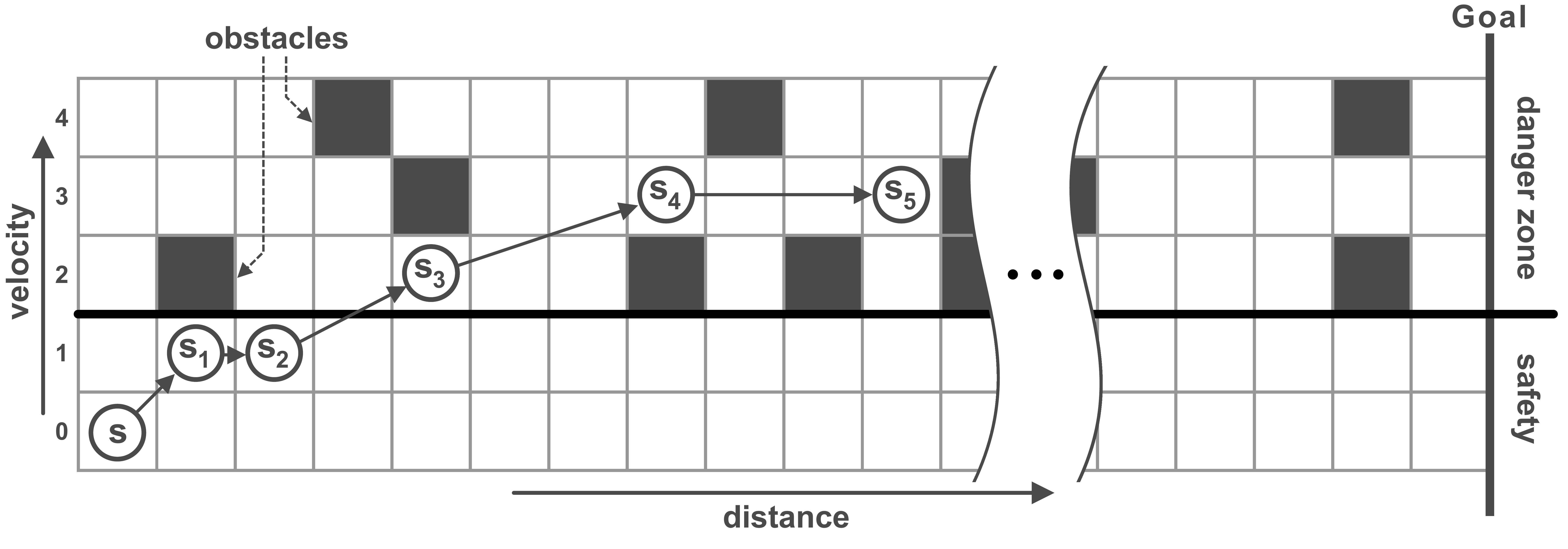}
\caption{Airspace domain.}
\label{fig:airspace}
\end{figure}

As an illustration of how to reduce the long term impact of actions in benchmark domains, we introduce a new benchmark domain called Airspace. In Airspace the state is two dimensional consisting of distance and altitude. The agent starts with zero speed and altitude. The goal is to traverse a predefined distance. In each step, the agent moves a distance matching its altitude. In other words, the higher the altitude the larger the agent's velocity towards the goal. Above altitude $1$, obstacles are blocking the way of the agent with uniform probability $p_{\mathit{obs}}$, thus making it more difficult for the agent to traverse the space with higher speeds. The agent has three possible actions: increase, keep or decrease the altitude. The agent can only take an action if it does not lead out of the bounds of the map and if the linearly interpolated state between the source and target state is not in collision.
Figure~\ref{fig:airspace} shows a sketch of the Airspace domain with an example path that demonstrates the dynamics of the agent. The agent starts from state $s$ in the bottom left corner and the goal is to cross the goal line on the right. The \textit{safety predicate} of the domain marks all states at altitudes $0$ and $1$ as safe. The \textit{safety heuristic} of a state is its altitude $- 1$. The heuristic function returns the remaining distance from a state to the finish line divided by the maximum speed.

\subsubsection{Properties of the Airspace Domain}

One of the key principles behind Airspace is to avoid allowing any single decision to have an outsized effect on a planner's overall performance.
%
Airspace exhibits this property because each altitude layer is connected with reasonable probability to the layers above and below. While, at any one time, some of these connections may be blocked by the obstacles, there will be enough options available that the agent still has the possibility of reaching any layer in the long term. 
Ideally, the agent is able to return to safe low altitudes regardless of its position as there are no long separators present in the state space between high and low altitudes.
In fact, using an obstacle probability of $p_{\mathit{obs}} 0.05$ allows a perfect agent to achieve a velocity of $~13$ on average, meaning the agent can vertically traverse the domain using only a small number of actions. We recommend choosing $p_{\mathit{obs}}$, the horizontal and vertical dimensions of the domain such that the number of steps it takes to vertically traverse the domain is negligible compared to the horizontal distance.  The overall effect is that the agent is only forced to fail if it makes a series of multiple poor decisions, rather than a single uninformed blunder.

Altitude layers higher than $1$ contain dead-end states due the velocity of the agent and the obstacles in the space. The probability of hitting an obstacle at altitude $a$ taking an altitude keeping action for any such layer is $p_\mathit{collision}^a = 1 - (1 - p_{\mathit{obs}})^\mathit{a}$. Thus, higher altitudes lead to better performance, as the agent travels faster towards the goal, but they are increasingly more difficult to navigate. It is not only more difficult to find a feasible path at high altitudes, but it makes it more difficult to complete a safety proof due both to the distance from the safe states and the probability of hitting an obstacle.

\begin{table*}[]
\small
    \centering
    \begin{tabular}{rlllllllllllllllll}
    \toprule
    altitude    &   3   &   4	&   5	&   6	&   7	&   8	&   9	&   10	&   11	&   12	&   13	&   14	&   15	&   16	&   17	&   18	&   19 \\ \midrule
    safety probability &   .95	&   .94	&   .89	&   .88	&   .86	&   .80	&   .74	&   .70	&   .64	&   .58	&   .51	&   .43	&   .35	&   .27	&   .19	&   .12	&   .06	\\
	proof length &   4	&	5	&	6	&	8	&	9	&	11	&	12	&	14	&	16	&	18	&	20	&	23	&	25	&	27	&	29	&	31	&	32	\\
	successful proof expansions &   4	&	5	&	7	&	8	&	10	&	12	&	14	&	18	&	21	&	26	&	32	&	38	&	45	&	53	&	60	&	65	&	68	\\
	failed proof expansions &	1	&	1	&	1	&	1	&	2	&	2	&	2	&	3	&	4	&	5	&	8	&	10	&	13	&	14	&	14	&	11	&	7\\\bottomrule
    \end{tabular}
    \caption{Difficulty of safety proofs in the Airspace domain ($p_{\mathit{obs}} = 0.05$) averaged over 100,000 states per altitude.}
    \label{tab:airspace}
\end{table*}

Table~\ref{tab:airspace} shows the probability of a state being safe, the average number of steps to reach a safe state, the average number of nodes expanded during both successful and failed proofs for each altitude. These average values were empirically measured from all states of a 100,000 long and 20 high Airspace domain with $p_{\mathit{obs}} = 0.05$.

Additionally, Airspace guarantees that the agent will not visit a state more than once, thus it eliminates scrubbing \cite{sturtevant:sdl}, focusing the benchmark on exploration and safety rather than on learning.


\section{Propagation of Dead-ends}

\begin{figure}[t]
\centering
  \includegraphics[width=0.8\linewidth]{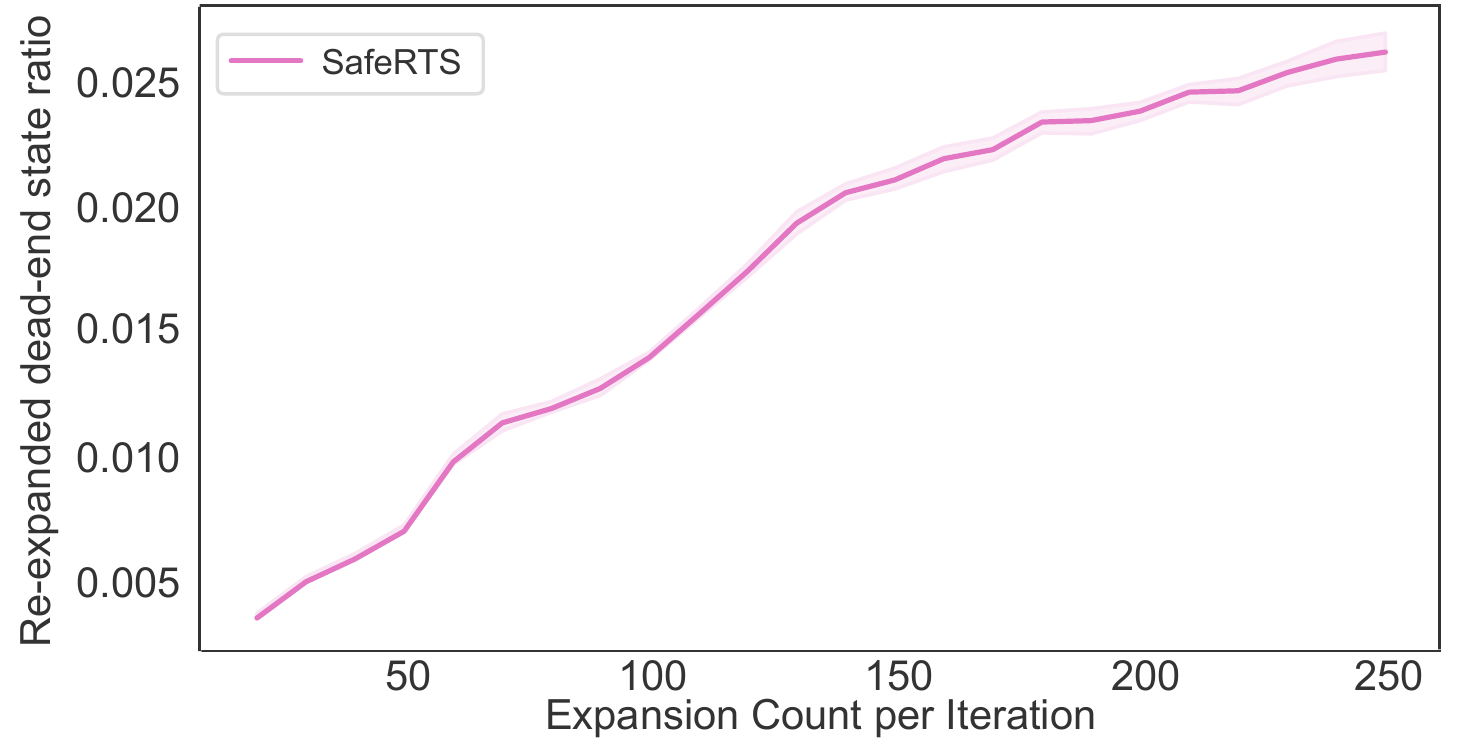}

\caption{Ratio between dead-end re-expansions and total number of node expansions.}
\label{fig:reexpansions}
\end{figure}

As our second contribution, we will introduce a method of using information gathered about dead ends to reduce future planning effort. Safety in real-time search relies on identifying the safety property of states using the safety predicate and propagating this information to predecessors. The converse of this problem can be formulated, given a predicate for dead-ends, as identifying and propagating dead-ends. Proving whether a state is safe or unsafe is as difficult as the original problem (e.g. consider a domain where the only safe states are the goal states). We argue that focusing on safety is more practical for planning than detecting dead-ends due to their propagation properties. While a state is considered to be implicitly safe if any of its successors are safe, a state can only be considered an implicit dead-end if all successors are proven to be dead-ends. Thus, having a safety predicate that is able to identify a subset of safe states is more practical than a similar dead-end detector.

Attempting to prove the safety property of a state has three potential outcomes. First, if a safe descendant is found then the state is safe. If the allocated budget for the safety proof is exhausted then the proof is non-conclusive: the state could be safe or unsafe. Lastly, when all descendants are expanded without leading to a safe state, then the state and all of its descendants are considered to be unsafe. We propose to cache this information to ensure that these states are not re-expanded in the future by the goal-oriented search effort or by the safety proofs. Marking states unsafe following a safety proof does not increase the space or time complexity of the algorithm, as all of the touched states have already been visited. It requires only an extra flag per state to be stored.

To assess the impact of this simple enhancement, we measure the ratio between dead-end re-expansions and the total number of expansions. In the augmented SafeRTS this quantifies the avoidable additional work.
Figure~\ref{fig:reexpansions} shows this ratio on the Airspace domain (band indicates 95\% confidence interval around the estimate). Eliminating the re-expansions would yield $0.5$ -- $2.5$\% additional expansions. The improvement did not lead to statistically significant results in our experimental setting. However, the yield is highly domain dependent. 
Our intuition is that the gain could be higher in domains where the agent has to revisit the same sub-spaces frequently and/or which contain large or recurring pockets of dead-end states.

\section{Proving Safety in Real-time Search}

Now we turn to our analysis of SafeRTS. As our third contribution, we prove that it is more efficient to allocate all resources for proving node safety at the end of the planning phase, rather than iteratively within it.

\begin{theorem} \label{thm:ez-proof}
Any safety proof that does not pass through the open list requires no expansions.
\end{theorem}
\begin{proof}
For every node $n$ internal to the Local Search Space, all immediate successors have been discovered. This follows from the fact that an LSS is constructed by expanding nodes on the frontier, at which point those frontier nodes become internal. A safety proof is a sequence of successor nodes, the last of which is a safe node. Any proof that does not pass through the open list must terminate in some successor that has already been expanded. Since the safety of all discovered safe nodes is propagated via a Dijkstra-like backup to all predecessors which are members of the LSS, no expansions are required to discover such a proof.
\end{proof}
\begin{theorem} \label{thm:effort}
Any safety proof that passes through the open list requires additional expansions.
\end{theorem}
\begin{proof}
If a safety proof passes through the open list, that means it is required to examine unexplored nodes which necessitates additional expansions.
\end{proof}


\begin{theorem} \label{thm:better-proof}
For each node $\mathrm{x}$ internal to the LSS (i.e. already closed) whose safety status cannot be proven without passing through the open list, there exists a node $\mathrm{y}$ on the open list such that proving the safety of $\mathrm{y}$ is strictly less difficult than proving the safety of $\mathrm{x}$.
\end{theorem}
\begin{proof}
If $x$ is not a dead-end \textit{w.l.o.g.}, let $\mathit{proof}^*(x) = \langle x, x_1, \dots, x_{k - 1}, y, \dots, x_{(k + l - 1)}, z \rangle$ where $x_i$ is the $i$th successor of $x$, $y$ is a node on the open list, and $z$ is a \textit{safe} node. Note that $y$ is $k$ state transitions away from $x$, and $z$ is $l$ transitions from $y$.
It is trivial to see that $\lvert\mathit{proof}^*(x)\rvert = k + l + 1$ and $\lvert\mathit{proof}^*(y)\rvert = l$.
$k$ and $l$ represent a number of state transitions, which by thair nature must be $\geq 0$. Since $x$ is not on the open list, $x \neq y \therefore k > 0$, meaning $|\mathit{proof}^*(x)| > |\mathit{proof}^*(y)|$.
Note that in the edge case where $y = z$, meaning $l = 0$, the above statement still holds.
\end{proof}

\begin{theorem}
Given a graph $\mathrm{G}$ with internal node $\mathrm{x}$ and frontier node $\mathrm{y}$, and given $\mathrm{proof(y)} \subset \mathrm{proof(x)}$, $\mathrm{proof(y)}$ is equivalently impactful as $\mathrm{proof(x)}$.
\end{theorem}
\begin{proof}
Let us say that $\mathit{proof}(x)$ and $\mathit{proof}(y)$ terminate in safe node $z$. Any proof $p$ that terminates in $z$ will result in the propagation of safety from $z$ extending back to all its predecessors including but not limited to both $y$ and its predecessor $x$, regardless of the start node in proof $p$. Therefore, the impact of $\mathit{proof}(y)$ and $\mathit{proof}(x)$ will be identical.
\end{proof}

\begin{theorem}
\label{theorem:safety_coverage}
Let $\mathrm{G_i}$ be a search graph expanded to $\mathrm{i}$ nodes, and let $\mathrm{G_j}$ be a subsequent search graph expanded from $\mathrm{G_i}$ to $\mathrm{j}$ nodes such that $\mathrm{j > i}$. Let $\mathrm{Pr_i \subset P}$ be any set of optimal proofs for all nodes in $\mathrm{G_i}$ that can be proven safe. Finding $\mathrm{Pr_i | G_j}$ requires equal or fewer expansions than finding $\mathrm{Pr_i | G_i}$.
\end{theorem}
\begin{proof}
We will first address the case where $j = i + 1$, then show that this proof extends to all cases.
First, the set of nodes $N$ that are internal nodes of $G_i$ (and hence also $G_j$) and whose proofs do not pass through $G_i$'s open list require no additional effort by Theorem \ref{thm:ez-proof}.

Now let us examine the extra node $n_j$ expanded in $G_j$. Note that since $n_j$ was just added to the search graph, it is on the open list of $G_j$. There are 3 possibilities:
\begin{enumerate}
    \item $n_j$ is part of a proof $\mathit{proof}^*(n_k) \in \mathit{Pr}_i$, but $n_j$ is not \textit{explicitly safe}. Theorem \ref{thm:better-proof} proves $\lvert\mathit{proof}^*(n_j)\rvert < \lvert\mathit{proof}^*(n_k)\rvert \; \forall \: k$ where $n_k$ is not on the open list. For any node $n_k$ on the open list whose optimal safety proof passes through $n_j$, the same principle applies in that $\lvert\mathit{proof}^*(n_k)\rvert > \lvert\mathit{proof}^*(n_j)\rvert$ by at least 1.
    \item $n_j$ is part of a proof $\mathit{proof}^*(n_k) \in \mathit{Pr}_i$, and $n_j$ is \textit{safe}. This is a special case of the above where $\lvert\mathit{proof}^*(n_j)\rvert = 0$
    \item $n_j$ is not part of any $\mathit{proof}^*(n_k) \in \mathit{Pr}_i$. In this case the expanded $G_j$ has no effect on the effort of computing $\mathit{Pr}_i$. Note that $n_j \notin G_i$, and therefore computing $\mathit{Pr}_i$ does not require us to find $\mathit{proof}^*(n_j)$.
\end{enumerate}
Clearly, the theorem holds for $i = 1$.  Now by induction, any two graphs $G_i, G_j : j > i, i \geq 1$ will display these same characteristics.
\end{proof}

SafeRTS interleaves exploration and safety proofs during its planning phase. As a direct consequence, it attempts safety proofs on nodes that become internal to the LSS by the end of the search iteration. As shown in Theorem \ref{theorem:safety_coverage}, it would be equally or less difficult to achieve the same or better safety coverage by doing safety proofs after the LSS expansions. SafeRTS has an anytime behavior but does not effectively utilize the real-time bound given.

\section{A Real-time Framework for Safety}

\begin{algorithm}[t]
\SetKwInOut{Input}{Input}
\Input{$\sroot, \mathit{iterationBound}, \mathit{explorationRatio} < 1$}
\textit{bound} $\leftarrow$ \textit{iterationBound} \label{alg:rtfs:bound} \\
\While{$\mathit{s_{root}} \neq \mathit{s_{goal}}$}{
\textit{explorationBudget} $\leftarrow$ \textit{bound} $*$ \textit{explorationRatio} \label{alg:rtfs:ebudget} \\
\textit{safetyBudget} $\leftarrow$ \textit{bound} $-$ \textit{explorationBudget} \label{alg:rtfs:sbudget} \\

\textit{lss} $\leftarrow$ \textsc{explore}$(\sroot, \mathit{explorationBudget})$ \label{alg:rtfs:explore} \\
\textsc{allocateProof}$(\textsc{selectTarget}, \mathit{lss, safetyBudget})$\label{alg:rtfs:safety}\\
\textsc{propagateH}$(\mathit{lss})$ \label{alg:rtfs:prop:h}\\
\textsc{propagateDeadEnds}$(\mathit{lss})$ \label{alg:rtfs:prop:dead}\\
\textsc{propagateSafety}$(\mathit{lss})$ \label{alg:rtfs:prop:safety}\\
$\starget \leftarrow$ \textsc{selectTarget}$(\mathit{lss})$ \label{alg:rtfs:selectTarget}\\
\If{$\starget$ is \textit{null}}{
\textsc{Terminate} $\lhd$ no safe path is available \label{alg:rtfs:terminate}\\ 
}
move the agent along the path from $\sroot$ to  $\starget$\label{alg:rtfs:move}\\
$\sroot \leftarrow \starget$\\
\textit{bound} $\leftarrow$ \textit{iterationBound} $+$ \textit{unusedBudget} \label{alg:rtfs:boundupdate} \\
}
\caption{Real-time Framework for Safety}
\label{alg:rtfs}
\end{algorithm}

Given this analysis of SafeRTS, we now introduce as our fourth contribution a general scheme called Real-time Framework for Safety (RTFS). This framework, shown in Algorithm~\ref{alg:rtfs}, composes an algorithm from four elements: a parameter that determines the ratio between goal and safety oriented search, and three main functions: \textsc{explore}, \textsc{allocateProof}, and \textsc{selectTarget}. First, we formalize our notion of these functions.

\begin{definition}
An exploration strategy is a function $\textsc{explore} : s_\mathit{root}, budget \rightarrow G$ that constructs a local search space given a root state and an exploration budget, returning a graph $G$ representing the Local Search Space.
\end{definition}

\begin{definition}
A safe target selection strategy is a function $\textsc{selectTarget} : G \rightarrow \langle n_1, n_2, \dots, n_k \rangle$ that defines a natural ordering for a set of nodes $N$ of size $k$ structured as a Local Search Space graph $G$.

\end{definition}

\noindent Note that when used to retrieve a single target $n_\mathit{target}$, the first node of the ordering is returned.

\begin{definition}
A safety proof allocation strategy is a function $\textsc{allocateProof} : \textsc{selectTarget}, G, \mathit{budget} \rightarrow \bot$ that allocates resources to proving the safety of nodes. It takes a safety target selection strategy, a search graph $G$, and a budget of time to allocate.
\end{definition}

\noindent Such a function is free to allocate resources in any way it chooses, but the $\textsc{selectTarget}$ parameter allows it to prioritize proof effort on nodes that will be chosen as the target to which the agent will commit.

RTFS exploits the real-time bound of the problem to pre-allocate the time spent on exploration and safety proofs [line~\ref{alg:rtfs:ebudget},~\ref{alg:rtfs:sbudget}]. RTFS takes the \textit{explorationRatio} as an input parameter. A higher value allows for more aggressive exploration, but decreases the likelihood of completing any safety proofs. The appropriate value should reflect the total available time per iteration and the difficulty of proving that a node is safe in a given domain.

The \textsc{explore} function defines the way the algorithm uses the expansion budget to build the local search space [line~\ref{alg:rtfs:explore}]. A trivial example of such function is A*, but
any exploration method that is capable of constructing a search tree could be used, such as wA*~\cite{pohl:hsv,rivera:iwi}, GBFS~\cite{pearl:his}, Beam search~\cite{russell:aim}, and Speedy~\cite{burns:hsw}. Using an exploration method that leads to a narrow and deep tree makes each safety proof more consequential as upon a successful proof every ancestor of the node will become safe.
    
Given a search tree, a \textsc{selectTarget} function, and an expansion budget, the \textsc{allocateProof} function distributes the given budget among the frontier nodes of the tree to prove their safety [line~\ref{alg:rtfs:safety}]. Using \textsc{selectTarget}, \textsc{allocateProof} can allocate resources based on the ordering provided. This function is highly non-trivial.

The learning function in line~\ref{alg:rtfs:prop:h} is the same as that of LSS-LRTA*. The dead-end propagation function in line~\ref{alg:rtfs:prop:dead} removes all nodes from the local search tree that were found to be unsafe or whose successors are all unsafe. Lastly, similar to SafeRTS the safety propagation function in line~\ref{alg:rtfs:prop:safety} marks as safe every node with a safe successor as discovered during exploration or as proven safe by \textsc{allocateProof}.

Given a search tree in which the safe and dead-end nodes are marked, the \textsc{selectTarget} function, in line~\ref{alg:rtfs:selectTarget}, selects a node which the agent should commit towards. The authors of SafeRTS described multiple examples for such target selection strategies and claimed best results with the previously described safe-towards-best.

The invocation of these functions might require less time than the given bound for the iteration.
The unused time is used towards the next planning iteration as shown in line~\ref{alg:rtfs:boundupdate}. Alternatively, in domains that do not allow such budget transfer, the \textsc{explore} and \textsc{allocateProof} functions are called with remaining budget distributed between them using the original ratio.

The structure of RTFS is designed to utilize the property proven in Theorem~\ref{theorem:safety_coverage}. The full LSS is constructed before any effort is spent on proving safety, which efficiently allocates available time such that it is more likely to prove more promising nodes than SafeRTS, as we will see below.

\section{Empirical Evaluation}


To ascertain the performance gain of RTFS, we create a configuration RTFS-0 with \textit{target selection} and \textit{safety proof allocation} functions that mimic SafeRTS.
SafeRTS allocates at least 50\% of its expansions towards the construction of the LSS, hence we set the \textit{exploration ratio} of RTFS-0 to $0.5$. Though both algorithms select the node on the open list with the lowest $f$ value at the time the proof is attempted, RTFS-0 differs from SafeRTS as it only attempts safety proofs after expanding the LSS. As such, its proofs are only limited by the iteration bound. If a target node is proven to be an implicit dead-end, RTFS-0 removes it and all other discovered dead-end nodes from the graph, then attempts to prove the next best node on the open list.


In our experiments we include two additional oracles, A* and Safe-LSS-LRTA*, to provide reference points. A* is executed offline and its execution time is not included in its GAT. This serves as a lower bound on the GAT and an upper bound on the velocity. Safe-LSS-LRTA* is a version of LSS-LRTA* that has access to an ideal dead-end detector, thus it only considers nodes that are safe. This offers the behavior of an agent-centered real-time search method that only has to focus on reaching the goal.


To evaluate the performance of our methods, we test them on the racetrack and traffic domains used by \citep{cserna:ade}, along with our new Airspace domain.  In racetrack, a derivative of the benchmark of \citep{barto:lar}, an agent with inertia and limited acceleration traverses a grid.  The agent's state is $\langle x, y, \dot{x}, \dot{y} \rangle$.  Dead-ends are reached if the agent collides with any blocked cell. Notably, the heuristic in this domain encourages the agent to achieve high velocity, as that produces the lowest estimate of GAT. 10 randomly sampled start positions were used for both instances we tested on. In the traffic domain, an extension of the domain used by \citep{kiesel:agq}, a agent moves in a grid, avoiding moving obstacles. A dead-end is reached if an obstacle collides with the agent before it reaches a goal state.
%
In the racetrack and Airspace domains, the planners select one action per planning iteration, while in traffic the planners commit to multiple actions to match the results presented in the SafeRTS paper. All tested planners were able to successfully avoid dead-ends in all benchmark instances, thus our results solely focus on performance indicators such as GAT and velocity.

First, we turn to the comparison of SafeRTS and RTFS-0. 
While both SafeRTS and RTFS-0 solved all instances of the traffic domain, the results are non-conclusive and stochastic. The GAT of both algorithms is highly fluctuating with no algorithm dominating the other.

\begin{figure}[t]
  \centering
   \vspace{-0.5cm}
  \includegraphics[width=0.8\linewidth]{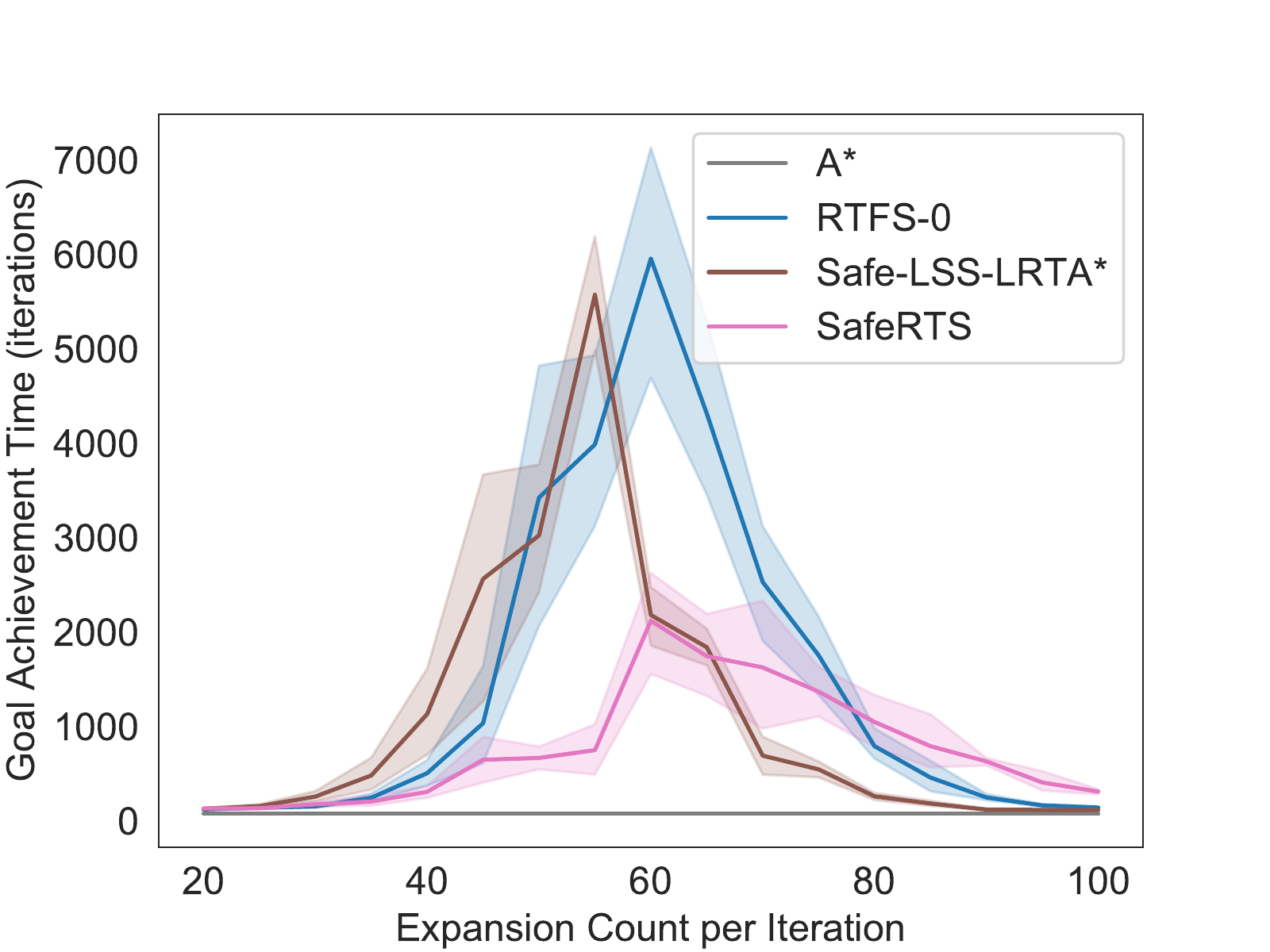}
  \caption{GAT on the Hansen--Barto racetrack.}
  \label{fig:results:racetrack:gat}
\end{figure}

\begin{figure}[t]
  \centering
   \vspace{-0.5cm}
  \includegraphics[width=0.8\linewidth]{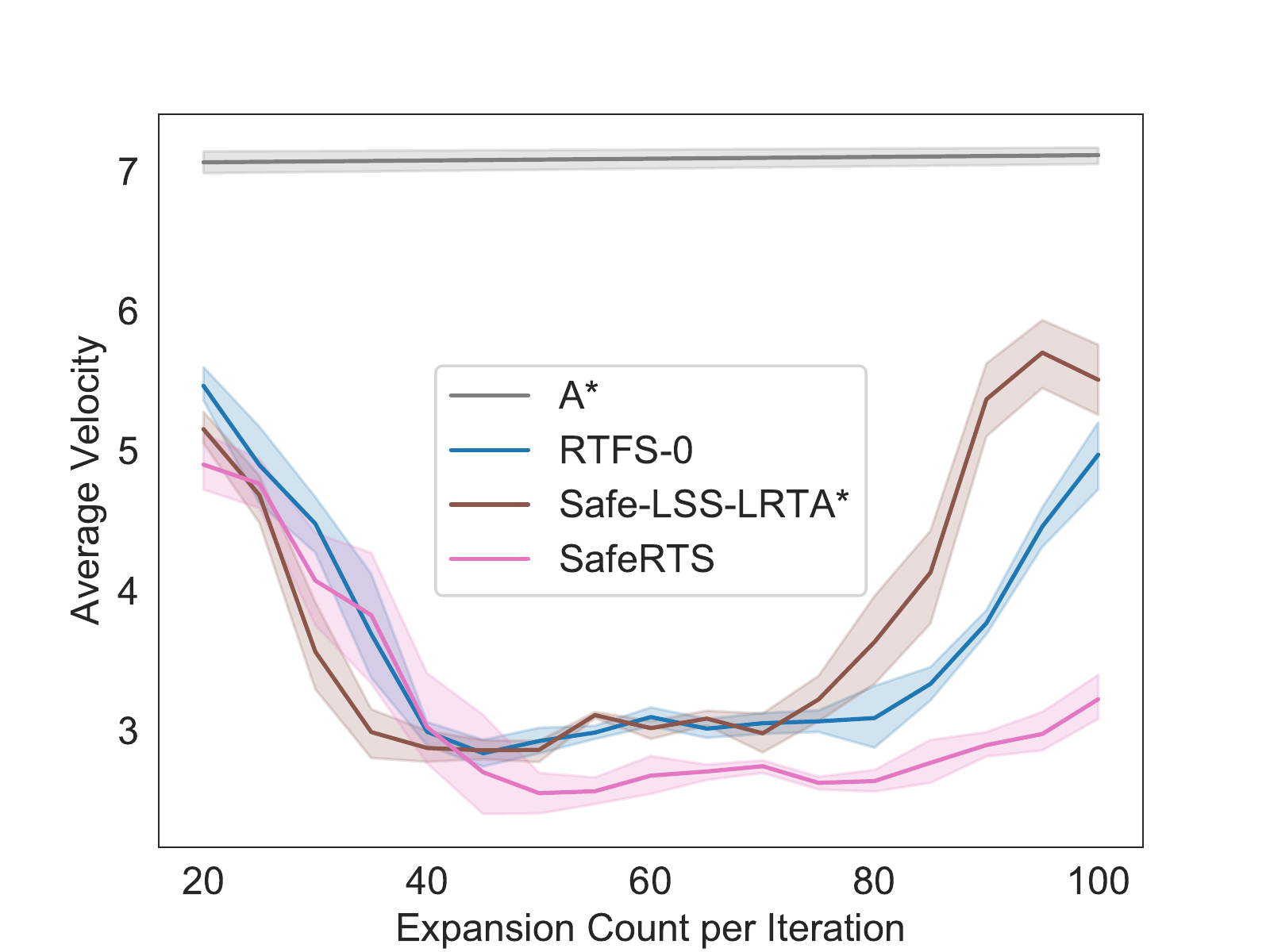}
  \caption{Average velocities on the Hansen--Barto racetrack.}
  \label{fig:results:racetrack:velocity}
\end{figure}

\begin{figure*}[t]
  \centering
   \vspace{-0.2cm}
  \begin{subfigure}{.33\linewidth}
    \centering
    \includegraphics[width=\linewidth]{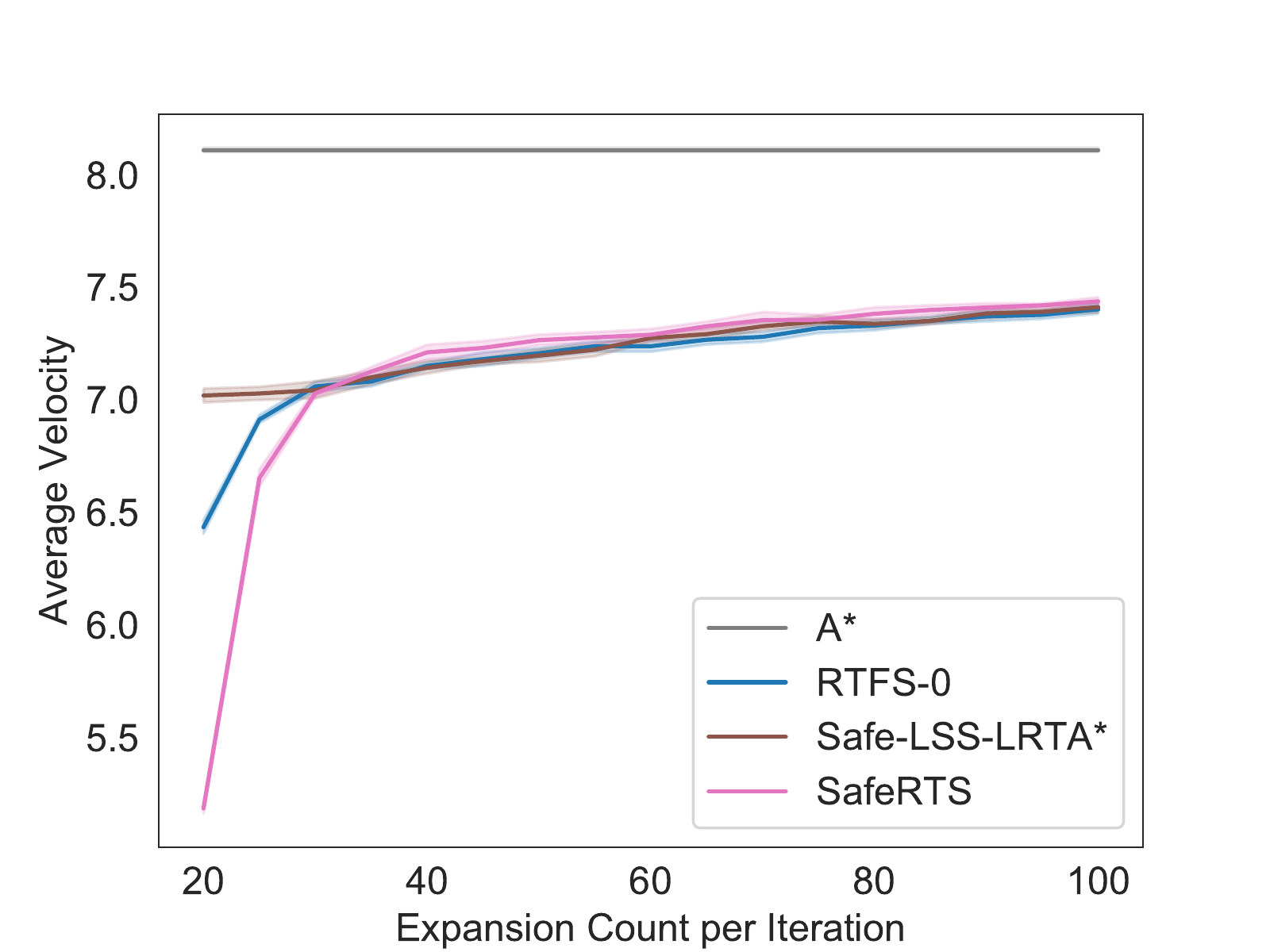}
    \caption{}
    \label{fig:results:airspace:05:limit10}
  \end{subfigure}
  \begin{subfigure}{.33\linewidth}
    \centering
    \includegraphics[width=\linewidth]{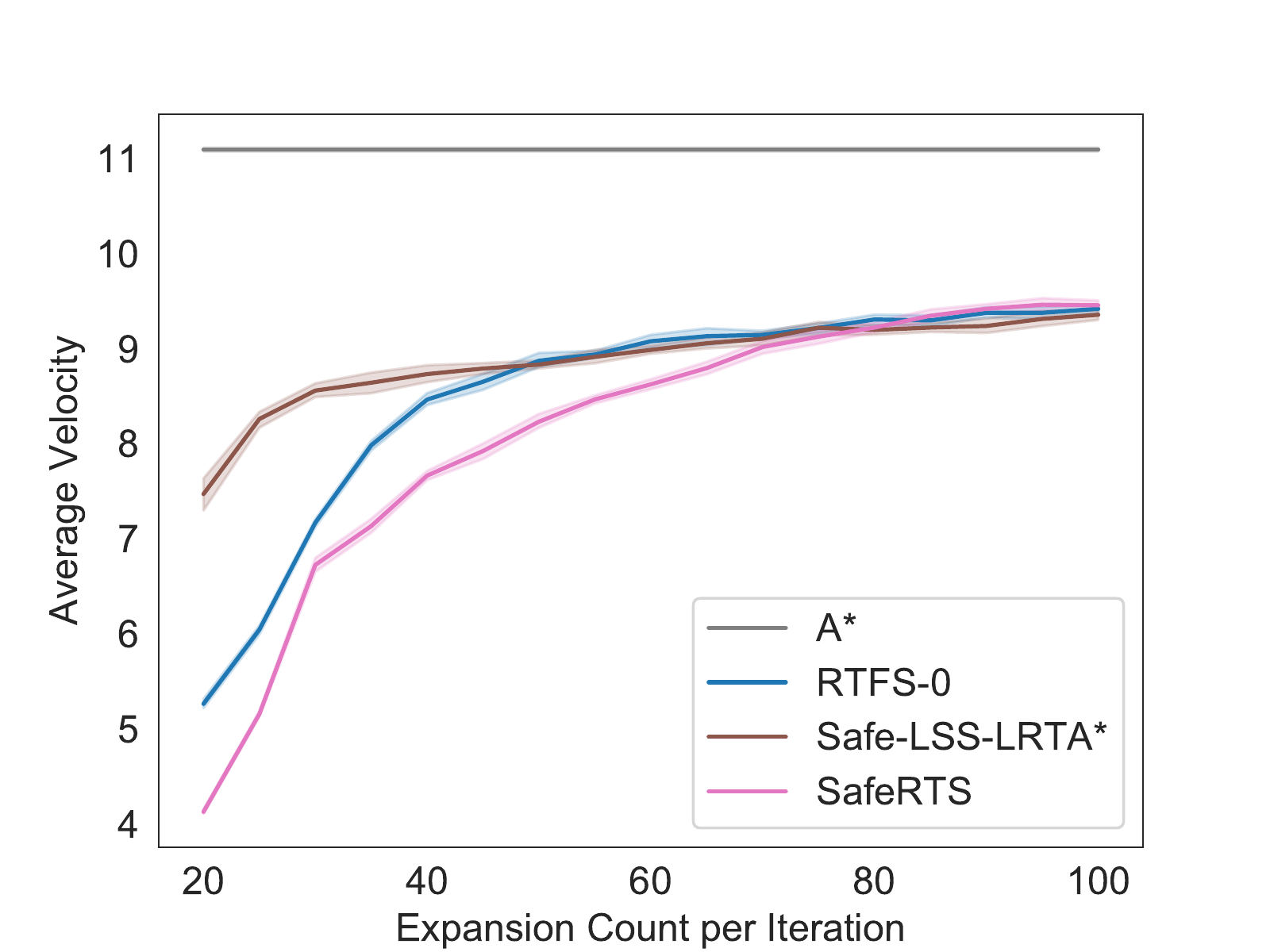}
    \caption{}
    \label{fig:results:airspace:05:limit14}
  \end{subfigure}
  \begin{subfigure}{.33\linewidth}
    \centering
    \includegraphics[width=\linewidth]{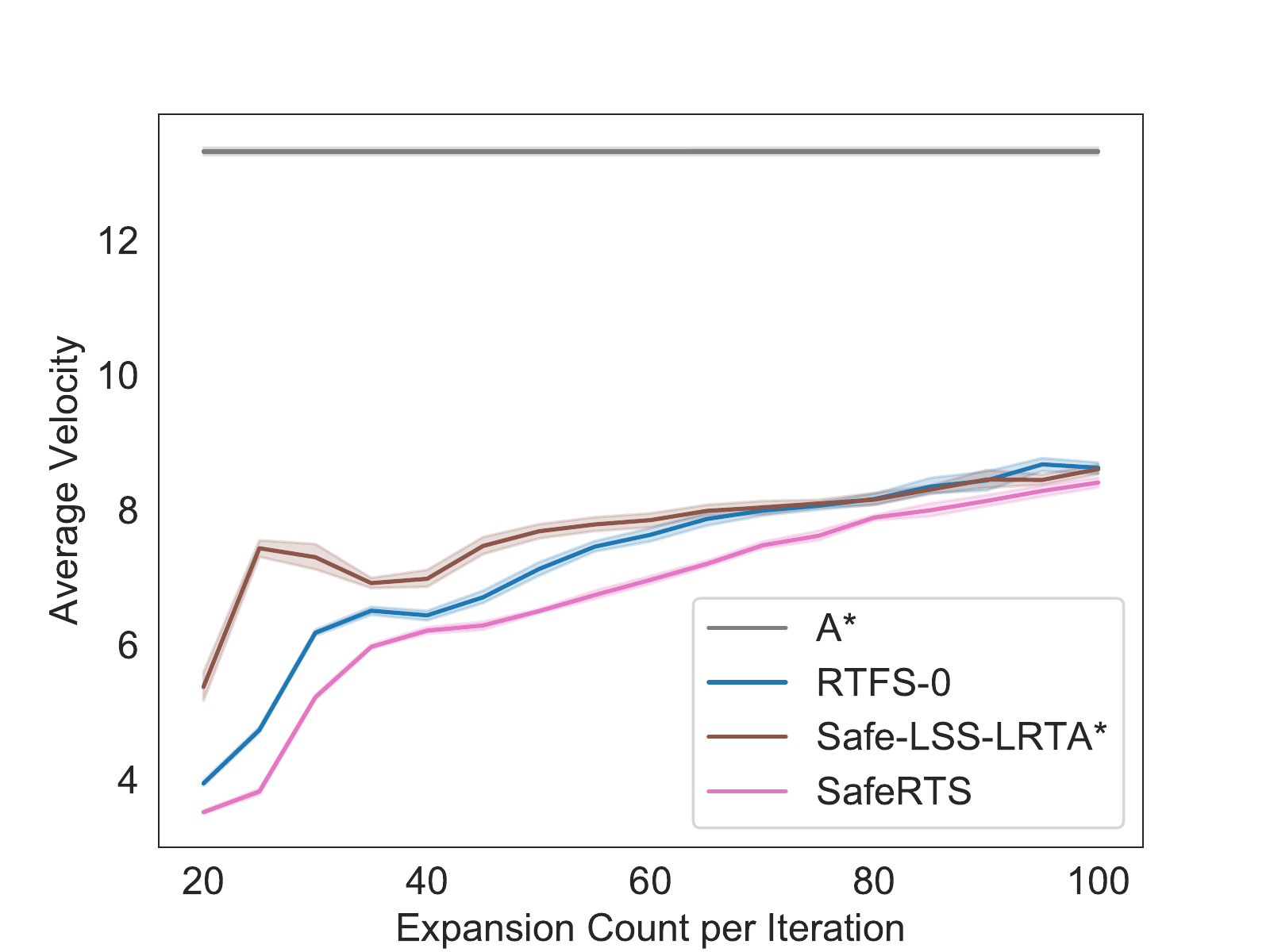}
    \caption{}
    \label{fig:results:airspace:05:limit20}
  \end{subfigure}

 \vspace{-0.2cm}
 
\caption{Average velocity on the Airspace domain of length 100,000 with altitude limits of 10 (a), 14 (b), and 20 (c).}
\label{fig:results:airspace:05}
 \vspace{-0.4cm}
\end{figure*}

\begin{figure*}[t]
  \centering
  \begin{subfigure}{.33\linewidth}
    \centering
    \includegraphics[width=\linewidth]{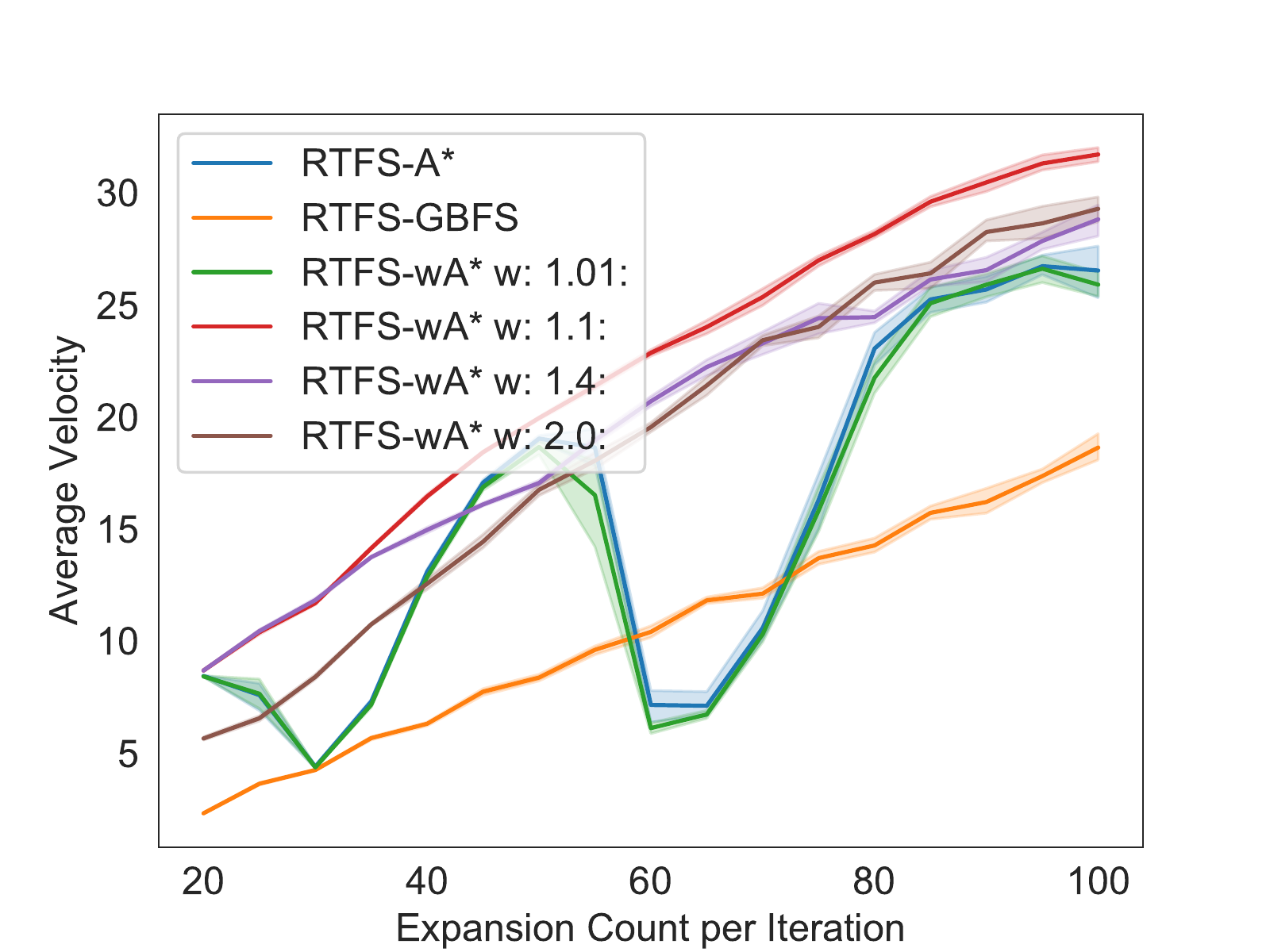}
    \caption{}
    \label{fig:results:airspace:001:explore}
  \end{subfigure}
  \begin{subfigure}{.33\linewidth}
    \centering
    \includegraphics[width=\linewidth]{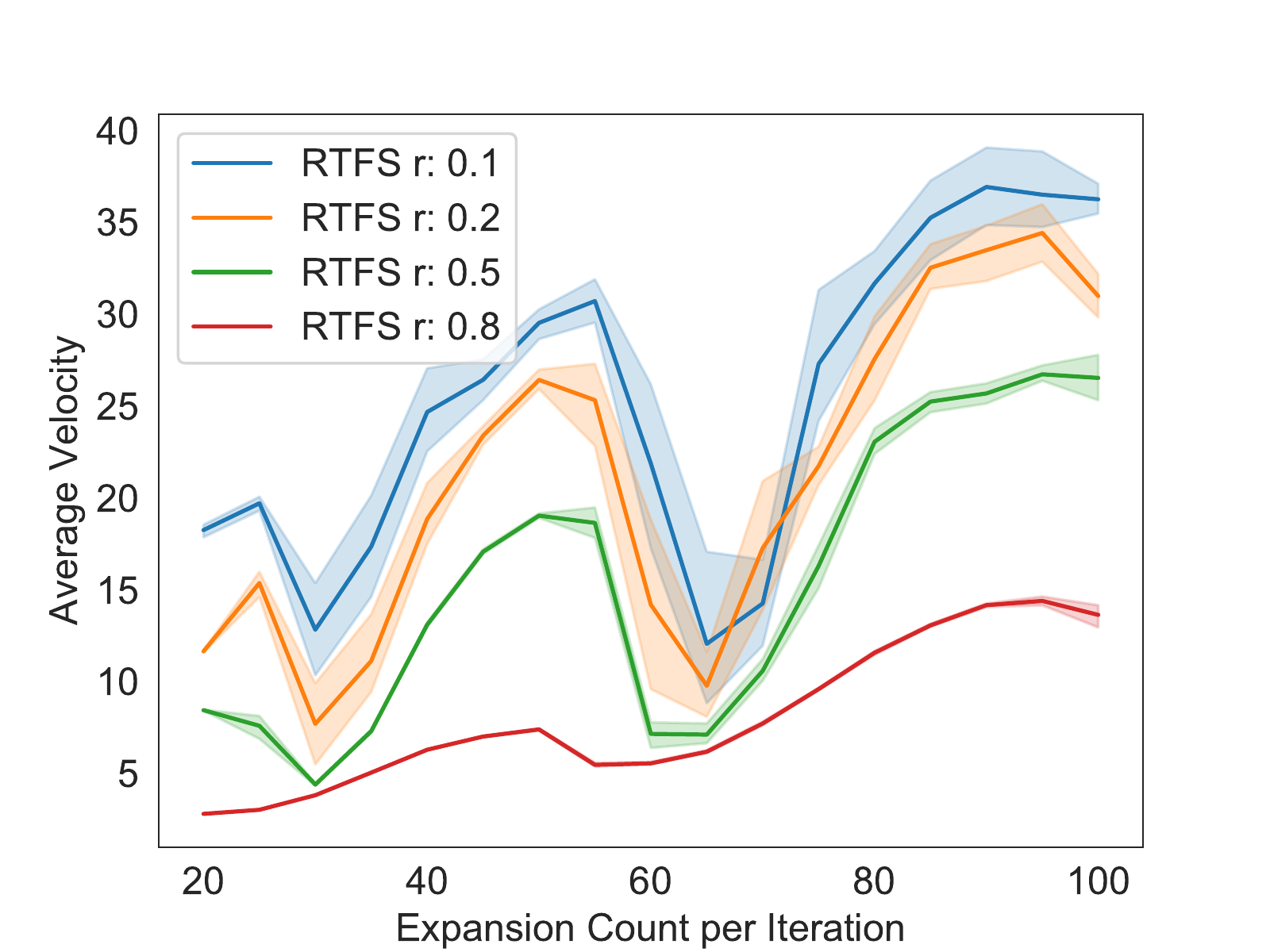}
    \caption{}
    \label{fig:results:airspace:001:ratio}
  \end{subfigure}
  \begin{subfigure}{.33\linewidth}
    \centering
    \includegraphics[width=\linewidth]{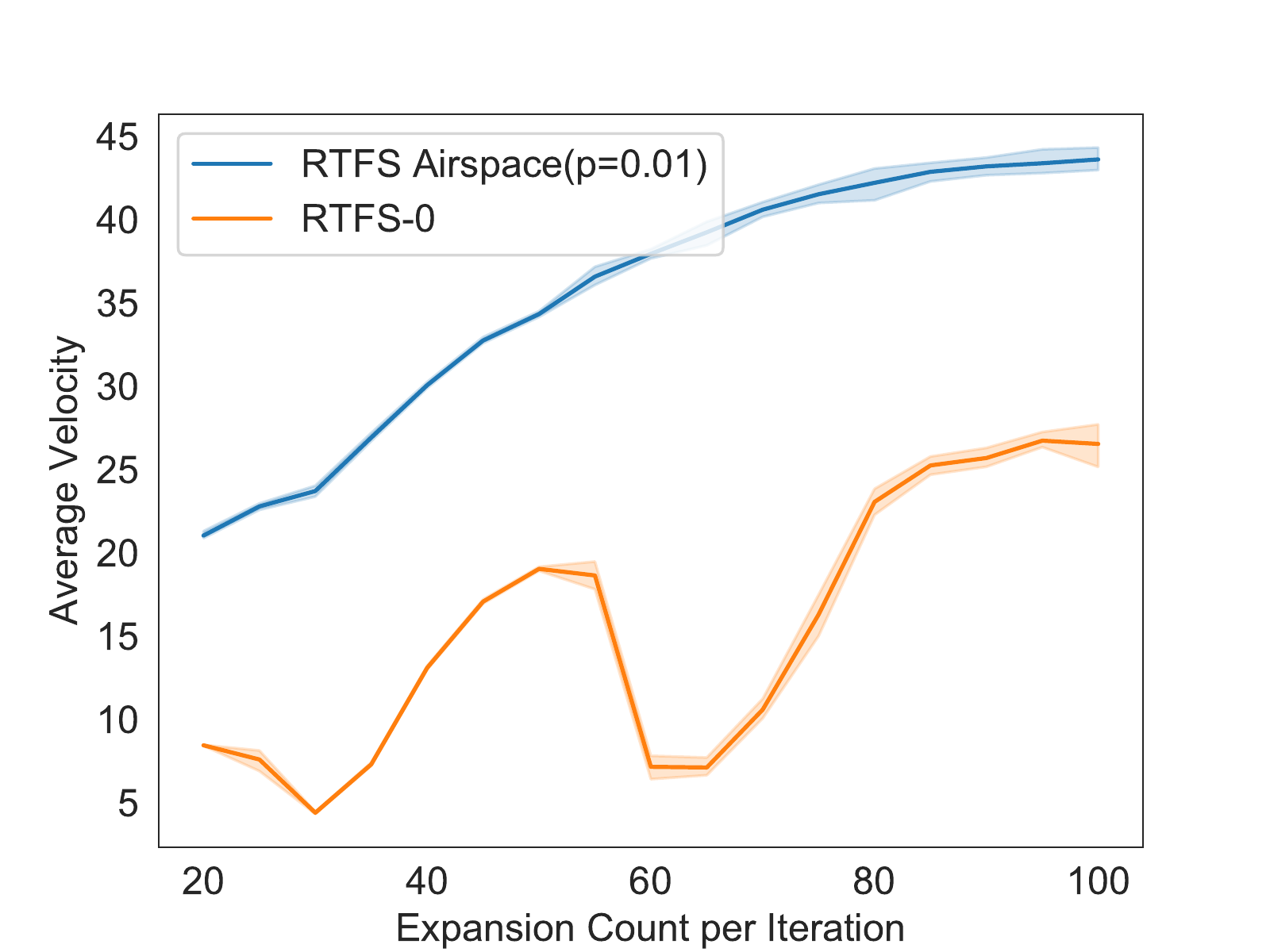}
    \caption{}
    \label{fig:results:airspace:001:composition}
  \end{subfigure}
  
 \vspace{-0.2cm}

\caption{Average velocity of RTFS variants on the Airspace domain of length 100,000 ($p_{\mathit{obs}} = 0.01$)}
\label{fig:results:airspace:001}

\end{figure*}

Figure~\ref{fig:results:racetrack:gat} shows the GAT on the Hansen-Barto racetrack domain \cite{cserna:ade}. All algorithms perform close to optimal with small action durations, however increasing the action duration drastically decreases the performance of all algorithms. The oracle-based Safe-LSS-LRTA* becomes $50$ times slower when more planning time is provided. We speculate that this is a result of actions with long term effects that lead to significant scrubbing \cite{sturtevant:sdl}. Empirically, the GAT in Figure~\ref{fig:results:racetrack:gat} imply that the agent revisited states many times before finding the path to the goal. The GAT of SafeRTS is superior to RTFS-0 and to Safe-LSS-LRTA* for low to medium lookaheads, which suggests that SafeRTS is simply a better heuristic for this particular domain. Figure~\ref{fig:results:racetrack:velocity} shows the average velocity of the agent, computed as the total distance travelled divided by the travel time. The velocity of RTFS-0 is higher than SafeRTS's and similar to that of the real-time oracle. Results on the uniform racetrack instance were similar thus those results are omitted.  
Both SafeRTS and RTFS-0 use the safe towards best target selection strategy.
Ultimately, this strategy is intended to commit the agent towards the top node on open if safety can be inferred. RTFS-0 managed to pick the top node in every single experiment on the Hansen-Barto racetrack, while SafeRTS selected targets of descendants of nodes with the average position on the open list of $7.2$ at action duration 20. Yet even though SafeRTS falls far short of selecting the top node, it performs better in this domain, which implies to us that this domain is not suitable for evaluateing the performance of safety oriented methods.


To further evaluate the performance of RTFS-0 and SafeRTS, We created 10 random Airspace instances with a horizontal distance of 100,000, maximum altitudes of 10, 14, 20, and $p_{\mathit{obs}} = 0.05$. In this domain we only focus on the average horizontal velocity as it directly corresponds to the GAT.

Finding a solution in Airspace is trivial as it can be completed by navigating the agent at the obstacle free altitude $1$. Finding a good solution is increasingly more difficult as the altitude limit increases. Figure~\ref{fig:results:airspace:05} shows the convergence of methods towards the oracle real-time search, and the average velocity shows a clear increasing trend as the time available per iteration increases. The convergence of SafeRTS and RTFS-0 slows down as the difficulty of the problem increases (Figure~\ref{fig:results:airspace:05:limit20}). RTFS-0 has faster average velocity and it closes the gap faster than SafeRTS.
%
%
We demonstrate the flexibility of RTFS by instantiating variants of it using two of the 4 degrees of freedom and evaluate them on 5 Airspace domains instances with a less dense obstacle probability of ($p_{\mathit{obs}} = 0.05$). The upper velocity bound achievable by A* is $70$. We investigate the effect of different exploration functions(A*, Weighted-A*, and Greedy Best First Search (GBFS)) as well as the impact of different \textit{explorationRatio}s.
Deviating from a locally optimal A* exploration of RTFS--0 can not only improve the average velocity but can reduce the variance of the outcome as shown on Figure~\ref{fig:results:airspace:001:explore}. In this context, RTFS--0 is denoted as RTFS-A*. The performance of RTFS-A* plummets periodically as the size of the search frontier and the available expansions are aligned in a way that a low performing node is selected. This is likely a consequence of selecting nodes from incomplete $f$ layers \cite{kiesel:agq}. Weighted A* seems to break this alignment by making the search tree deeper. Additionally, the overall performance improved up to $20$\% when a weight of $1.1$ is used. Further increases in the weight reduced the average velocity, converging the performance of GBFS, which achieved the lowest average velocity in our experiments.
In addition to the exploration function, we tested a range of \textit{explorationRatio}s, that determines how much time should be spent on exploration and safety proofs in each iteration. \textit{explorationRatio} $= 0.1$, or $r = 0.1$ for short, means that 10\% of the time is spent on expansions. Figure~\ref{fig:results:airspace:001:ratio} shows that decreasing the exploration time increased the performance as higher altitudes require longer and more difficult safety proofs. However, it also amplified the fluctuations discussed above.
Lastly, the performance of RTFS composed from the wA* ($w = 1.1$) exploration function and an \textit{explorationRatio} of $0.1$ is shown on Figure~\ref{fig:results:airspace:001:composition}. The remaining function are the same as in RTFS-0. While the union of these modifications increases the performance by 100\%, it only demonstrates the flexibility of RTFS and it is not intended to serve as a general recommendation of this particular configuration.

\section{Discussion}


While the above results are encouraging, it is important to note that simplistic target selection and safety allocation strategies were used in RTFS with the intention of matching SafeRTS for better comparison.
The topic of selecting the node to commit to is an issue fundamental to online real-time planning, and a deep investigation is outside the scope of this paper, beyond ensuring that the node selected is \textit{safe}. Resource allocation for safety is similar to the problem of a parallel portfolio of algorithms: we may choose any of a number of promising frontier nodes on which to attempt a proof. The safety allocation problem may have additional constraints in that we prefer to prove nodes based on the ordering provided by the target selection strategy, but it is not always clear when a proof of a node should be abandoned or not attempted in the first place in favor of some other promising node which may be easier to prove.
Learning based methods have been proposed to address these settings \cite{oceallaigh:mrh,cserna:ptt,petrik:lppa}.

\section{Conclusion}

This work has four contributions. First we introduced a new domain with dead-ends called Airspace that minimizes the long term effects of actions and was designed specifically to evaluate safe real-time methods.
Second, we showed that the simple method of caching dead-ends provided a mild performance improvement.
Third, we proved that proving safety is more effective when it is done after expanding the local state space. Lastly, we combined these finding into a flexible planning framework, RTFS, to address real-time planning in the presence of dead-ends.
We demonstrated that, when configured like SafeRTS, RTFS provides improved performance. Unlike SafeRTS, RTFS can be tuned to each domain to achieve higher performance, and thus is a flexible model for safe real-time search.

We hope this work encourages further research on avoiding dead-ends in the online planning setting.


\bibliographystyle{aaai}
\bibliography{master}

\end{document}